%% file: paper.tex
\documentclass[letterpaper, 10 pt, conference]{ieeeconf}  % Comment this line out if you need a4paper

\IEEEoverridecommandlockouts                              % This command is only needed if 
\let\proof\relax                % Resolve the "proof already defined"
\let\endproof\relax

\makeatletter
\let\NAT@parse\undefined
\makeatother

\usepackage{times}
\usepackage{graphicx}
\usepackage{etoolbox}
\usepackage{bm}
\usepackage{xspace} 
\usepackage{standalone} 
\usepackage{pgfplots} 
\usepackage{filecontents} 
\usepackage{tikz} 
\usepackage[T1]{fontenc} 
\usepackage[font=small]{caption}
\usepackage{subcaption}
\usepackage{mathtools}
\usepackage{enumerate}
\usepackage[ruled,vlined,noend]{algorithm2e}
\usepackage{amssymb,amsmath,amsthm}
\usepackage{todonotes}
\usepackage{comment}
\usepackage{ifthen}
\usepackage{siunitx}
\usetikzlibrary{math}
\usepackage{transparent}
\usepackage{sidecap}
\newtheorem{theorem}{Theorem}

\newenvironment{sproof}{%
  \proof}{\endproof}

\input{defs.tex}

\usepackage[numbers]{natbib}
\usepackage{multicol}
\usepackage{tabularx}
\usepackage[bookmarks=true]{hyperref}
\usepackage{xcolor}
\hypersetup{
  colorlinks,
  linkcolor={red!50!black},
  citecolor={blue!50!black},
  urlcolor={blue!80!black}
  }

\def\BibTeX{{\rm B\kern-.05em{\sc i\kern-.025em b}\kern-.08em
    T\kern-.1667em\lower.7ex\hbox{E}\kern-.125emX}}

\theoremstyle{plain}
\newtheorem{thm}{Theorem} 
\theoremstyle{definition}
\newtheorem{prop}{Theorem} 
\newtheorem{propp}{Theorem} 
\newtheorem{definition}[thm]{Definition}
\newtheorem{property}[prop]{Property}
\newtheorem{proposition}[propp]{Proposition}

\newcommand\blfootnote[1]{
	\begingroup
	\renewcommand\thefootnote{}\footnote{#1}
	\addtocounter{footnote}{-1}
	\endgroup
}
\begin{document}

%\title{Does the divulged plan hurt user privacy?
%\title{Privacy preserving robot design
%\title{Automated sensor design and planning for privacy preserving robots
%\title{Privacy preserving robots: plan, sensor, and actuator designs jointly 
%\title{Privacy preserving robots: planning jointly with sensor and actuator design
%\title{Finding plans and robot designs for privacy preservation
%\title{Privacy by design: designing robots and finding plans that protect privacy
%\title{Sensors which suffice for some robotic planning problem
%\title{Exploring all robotic sensors sufficient to solve planning problems
%\title{Finding sensor designs for robotic planning
%\title{Abstractions for exploring all robotic sensors that suffice to solve a planning problem
\title{\LARGE \bf Abstractions for computing all robotic sensors that suffice to solve a planning problem
}
% TODO: Just sensor design? What about actuator?

\newcommand\unsure[1]{\textcolor{brown}{{#1}}}
\newcommand\oldtext[1]{}

\author{
Yulin Zhang and Dylan A. Shell\vspace*{-5pt}\thanks{Yulin Zhang and Dylan A.
Shell are with the Dept. of Computer Science\gobble{ and Engineering}, Texas A\&M
University, College Station, TX, USA.  %\{yulinzhang\,|\,dshell\}@tamu.edu%
}
}

\maketitle

\oldtext{
To limit the information that robots may disclose to untrusted parties, this
paper explores planning for robots that are constrained in what they may know.
The idea of this paper is to automatically produce robot designs (descriptions
of sensors and actuators) and also to find plans for a robot to achieve goals
while subject to informational stipulations.  The stipulations
restrict what can be inferred from the robot's state and behavior so that, even
if the robot is compromised by an adversary, there are definite limits on
what may be learned.
We introduce algorithms that help characterize the space of designs,
permitting compromises and tensions to be visualized, and also optimization of
trade-offs that balance design and execution costs.
Prior examples of robots that actively manage their ignorance have
used certain specific, one might even say peculiar, sensors.  A detailed
understanding of what attributes these sensors possess has been unclear; 
this paper sheds some light on the matter.
The paper presents examples, mostly of a didactic form, examined with the help
of our implementation.  }

\begin{abstract}
Whether a robot can perform some specific task depends on several aspects,
including the robot's sensors and the plans it possesses.  We are interested in
search algorithms that treat plans and sensor designs jointly, yielding
solutions---i.e., plan and sensor characterization pairs---if and only if they
exist.  Such algorithms can help roboticists explore the space of sensors to aid
in making design trade-offs.  Generalizing prior work where sensors are modeled
abstractly as sensor maps on p-graphs, the present paper increases the potential
sensors which can be sought significantly.  But doing so enlarges a problem
currently on the outer limits of being considered tractable. Toward taming this
complexity, two contributions are made: (1) we show how to represent the search
space for this more general problem and describe data structures that enable
whole sets of sensors to be summarized via a single special representative; (2)
we give a means by which other structure (either task domain knowledge, sensor
technology or fabrication constraints) can be incorporated to reduce the sets to
be enumerated.  These lead to algorithms that we have implemented and which
suffice to solve particular problem instances, albeit only of small scale.
Nevertheless, the algorithm aids in helping understand what attributes sensors
must possess and what information they must provide in order to ensure a robot
can achieve its goals despite non-determinism.  
\end{abstract}

\blfootnote{This work was supported by the NSF through awards
\href{http://nsf.gov/awardsearch/showAward?AWD_ID=1453652}{IIS-1453652} and
\href{http://nsf.gov/awardsearch/showAward?AWD_ID=1527436}{IIS-1527436}.
}

\vspace*{-2pt}
\section{Introduction}
\vspace*{-2pt}

We currently approach robot sensors from the perspective of consumers,
purchasing whatever seems necessary from a catalogue, then writing program code
to make robots useful. This perspective puts practical constraints up front: it
is influenced by technologies that are currently available, it limits options to
what can be fabricated cheaply and sold profitably. Worse, it relies
on roboticists to reason (often only heuristically) about the information
needed for a robot to achieve its goals. If there is some notion of task
structure, reasoning about it is seldom formalized, and may be tied to
assumptions often taken for granted (e.g., for fixed price, greater sensor
precision is better).  This paper approaches the question of sensors from a more
fundamental perspective---asking how we might represent and explore 
\emph{conceivable} sensors. It is, therefore, of a more theoretical nature.
\begin{figure}[ht!]
\vspace*{-8pt}
\centering
\includegraphics[scale=1.2]{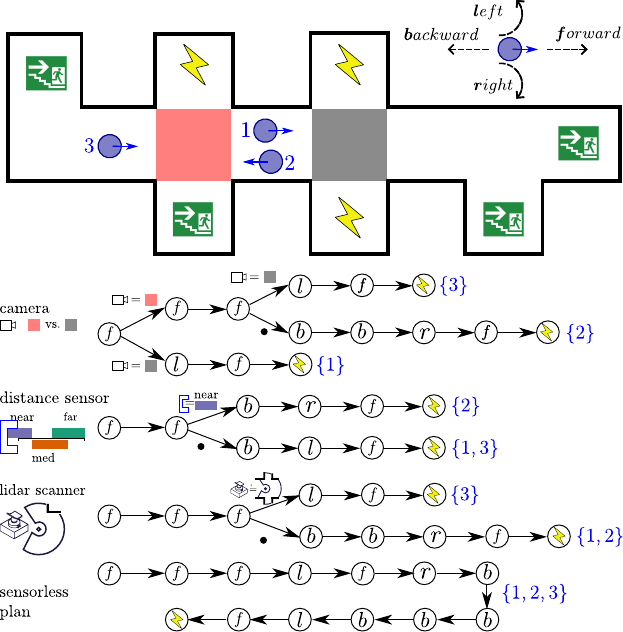}
\caption{A wheeled robot (as a blue disk) needs a
charging station (the lightning bolts), but is slightly lost (the uncertainty
in its initial pose is shown visually, as three possibilities).  Unable
to navigate stairs, it must avoid those locations lest it topple down a
stairwell.  The robot is able to recharge its battery despite the presence of
uncertainty, with the help of
either a camera, a simple linear distance sensor, or a short-range  
scanning lidar. (If bumping into walls is permitted, a sensorless plan is possible as well.)\label{fig:motivating_example}}
\vspace*{-22pt}
\end{figure}

%
% uncertainty which we model via non-determinism) 
% interested in "decision problem" form

% Privacy
% - Use text below

% Contributions of this paper:
% - general sensors
% - fabrication constraints.

% tied to technologies.
% fails to look at task structure.
% baked in assumptions: more information = better
% isn't very enlightening about what is needed

% Joint-ness = plans and sensors together, 
% The point is that our perspective emphasizes the sensors.
% Because they can be helpful, and helpful in this particular way.
Which sensors are necessary depends on what your robot wants to do.  We
study robots that act to attain goals while managing
uncertainty, formulating these precisely as planning problems, under
worst-case non-determinism.  Unlike many papers
entirely focused on finding plans, this paper examines ways in which sensors
affect whether a planning problem can be solved. The perspective is that sensor
choices alter the set of feasible plans, and we look at sensor/plan pairs
jointly.  We examine the space of sensors that are useful with
respect to a specific given problem.  These sensors, indeed especially those that
provide little information, can be enlightening.  Still, we do require they
provide information to make progress toward
goals~\cite{erdmann95understanding}, even in the presence of uncertainty. We
are interested in exploring all sensors, including even hypothetical ones, for
which there exists some goal-achieving plan.

Fig.~\ref{fig:motivating_example} shows a simple didactic
scenario illustrating multiple aspects of the  problem: a robot, uncertain about
its initial position and incapable of navigating stairs, needs to reach a
charging station.  
%We are interested in finding all sensors that could usefully ensure goal attainment. 
%We give four exemplar sensors that solve the problem under different plans:\\
We give four exemplar sensors that, under different plans, ensure goal attainment:\\
\phantom{.}~($i$)~a camera to distinguish red and gray 
helps to eliminate uncertainty in the initial pose when following the top plan;\\
\phantom{.}~($ii$)~a robot with a distance sensor can disambiguate initial position $2$
from $\{1,3\}$, since it observes that it is near the wall after two forward
moves only when it starts at $2$, while observing medium or far 
from the wall for $\{1,3\}$; \\
\phantom{.}~($iii$)~with a lidar sensor the robot can
distinguish $3$ from $\{1,2\}$ since, after three forward moves from $3$, it
senses a different polygon from those of $2$ and $3$.\\
\phantom{.}~($iv$)~the vacuous sensor also suffices, albeit only under the
assumption of benign collisions, and 
with many steps.

The sensors do not all quash the uncertainty completely, but they eliminate
enough to reach the goal under different plans. For example, the robot with a
distance sensor never resolves whether it came from $1$ or $3$ in executing the
corresponding plan. The robot with a lidar sensor does not distinguish $1$ from
$2$. But, in both cases, the robot reaches a charger.  There are also important
differences in the sensors' fidelity. The camera divides all the locations into
three equivalent classes: a red location, a gray one, and the white ones. In
contrast, the distance sensor's specification tells us that middle range
distance readings are noisy, failing to separate medium and far distances from
the wall crisply (when the robot observes `med', then it is either at a medium
or a far range from the wall; when obtaining `near', it is close to the wall). 

Sensors can be modeled as the information they provide for the plan.  While
previous
works~\cite{ghasemlou2018delineating,ghasemlou2019accel,zhang18discreteplan}
regard sensors as partitions over all events to be perceived, this paper is
more general, considering sensors as covers.  Doing so requires some care,
including new representations and means to lessen the combinatorial explosion
that a na\"\i ve treatment entails.

\gobble{
Based on the cover
representation, we proposed approaches to search for all sensors jointly with
plans in robot's belief tree. A notion of upper cover is introduced for a
compact representation, which could potentially lead to a speed up of the
algorithm. On the other hand, we may start with some fabrication constraints,
e.g., we only get to use timers to realize a VHF omnidirectional range (VOR)
sensor which measures the bearing.  The fabrication constraints describe a set
of sensors that can be realized. We are only interested in those that can both
be realized and reach the goal. To address these fabrication constraints, we
introduced several properties on the covers, and integrated them into the
search algorithm.  
}

\section{Model}

We study a setting depicted in Fig.~\ref{fig:modeloverview}. 
The {robot} is equipped with a \emph{sensor}, through which it receives
observations from the {world}.  Actions are chosen to alter states according to
the robot's \emph{plan} to, ultimately, reach some goal states in the world.
The sensor may have limited fidelity and fail to distinguish different
observations from the world.  The uncertainty in sensing is modeled via a type
of function, termed a \emph{sensor map}.  These elements are formalized in terms
of p-graphs and sensor maps that we outline below.

\vspace*{-8pt}
\begin{figure}[ht!]
\centering
\includegraphics[scale=0.8]{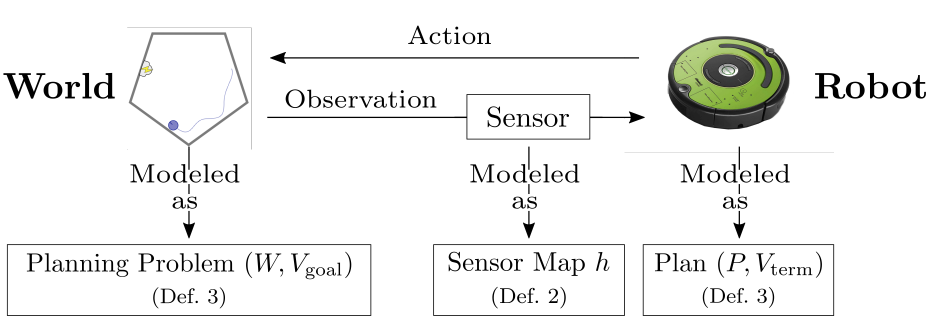}
\caption{An overview of the setting: the robot is modeled abstractly as realizing
a plan to achieve some goal in the world. The sensor is modeled as a sensor map.
Both the world and the plan have concrete representations as p-graphs.%
\label{fig:modeloverview}}
\vspace*{-12pt}
\end{figure}

\subsection{A p-graph and its interaction language}
We start by defining p-graphs and related properties.
For more complete formal definitions, please see~\cite{saberifar18pgraph}.
\begin{definition}[p-graph~\cite{saberifar18pgraph}]
A \defemp{p-graph} is an edge-labeled directed bipartite graph {\small $G=(V_y\cup
V_u, E, Y, U, V_0)$}, where 
\begin{tightenumerate}
\item the finite vertex set $V(G)$, whose elements are also called
\emph{states}, form two disjoint subsets: the
\emph{observation vertices} $V_y$ and the \emph{action vertices} $V_u$, with
$V(G)=V_y\cup V_u$;
\item each edge $e \in E$ originating at an observation vertex bears a set of
observations ${Y(e) \subseteq Y}$, containing \emph{observation labels}, and
leads to an action vertex; 
\item each edge $e \in E$ originating at an action vertex
bears a set of actions ${U(e) \subseteq U}$, containing \emph{action
labels}, and leads to an observation vertex; and
\item a non-empty set of states $V_0(G)$ are designated as \emph{initial states},
which may be either action states ($V_0(G)\subseteq V_u$) or
observation states ($V_0(G)\subseteq V_y$), exclusively.
\end{tightenumerate}
\end{definition}
We use the word `event' to mean either an action or an observation. 
Here, $Y$ is assumed to be finite.
%except in Section~\ref{sec:noiseless} (but also cf.~\cite{setlabelrss}).
A sequence of alternating actions and observations, also called an
\emph{execution}, can be traced in the
p-graph, if there exists a path, starting from some initial state, with the
same number of edges and each edge in the path bears the corresponding event.  A
p-graph $G$ describes a set of such event sequences. 
%that can be traced in the p-graph, which are called the \emph{language} of $G$,
%denoted $\Language{G}$.

One p-graph is used to model the world, another the robot.
\oldtext{
In the latter p-graph, the vertices constitute the
state that is stored, acted upon, and/or represented. 
Hence, the set of
vertices, $B\subseteq V(G)$, eventually reached by tracing each execution from
some initial vertex, establishes a equivalence relation: the two executions
that eventually reach the same set of vertices are deemed indistinguishable. 
We denote the set of executions that finally \emph{reaches exactly} $B$ with
$\exactreachings{G}{B}$. All executions in $\exactreachings{G}{B}$ are
considered indistinguishable in graph $G$. The set of executions which \emph{reaches}
vertex $v$ is denoted $\reachings{G}{v}$.
Note that, each execution in $\exactreachings{G}{B}$ should reach all vertices
in $B$ and should not reach any other vertex than those in $B$; while each
execution in $\reachings{G}{v}$ may also reach other vertices in addition to $v$.} 
Both p-graphs are coupled, resulting in a planning problem.
Sensors influence this coupling relationship by influencing the
distinguishability of observations made by the robot.  Conflations and
corruptions of events are treated next.
\oldtext{
By tracing the robot's perceived executions in both p-graphs, 
a correspondence can be established from robot states to world states, i.e., 
a relation tracking what the robot knows about the world.
Sensors influence this correspondence by influencing the distinguishability of
observations made by the robot.  Conflations and corruptions of events are
treated next.
}

\subsection{Sensor maps}
\begin{definition}[observation/sensor maps~\cite{setlabelrss}]
A \defemp{sensor map} on p-graph $G$ is a function $h: Y\to
\powerset{X}\setminus\{\emptyset\}$ mapping from an observation in $Y$ to a
non-empty set of observations $X$, where $\powerset{X}$ is the powerset of $X$.
\end{definition}

\oldtext{
One action can be generated by multiple realizations or commands from the
actuator. One command could generate multiple actions. When an actuator executes
commands from $\cap_{i=0\dots m}h(u_i)$, the outcome is treated as a
non-deterministic choice among $\{u_i\}_{i=0\dots m}$.  If a command $x$ unique
to $u_i$, i.e. $x\in h(u_i)\setminus\cup_{j\neq i} h(u_j)$, then the robot knows
that only action $u_i$ will be triggered.}

If $h$ maps $y_1$ to $\{x_1, x_2, x_3, x_4\}$ then, when event $y_1$ happens
in the world, the robot may receive any of those four values as a sensor
reading; further, we assume the choice happens non-deterministically.

%Unlike the actuator, the sensor does not get to choose from its
%readings. Hence, to make two observations indistinguishable through the sensor,
%we have to make sure that they are indistinguishable under every possible sensor
%reading, i.e.  $h(y_i)=h(y_j)$. If $h(y_i)\cap h(y_j)=\emptyset$, then the robot
%is always able to distinguish $y_i$ from $y_j$.

Given any subset of sensor readings $X'\subseteq X$ as input to (that is,
observed or perceived by) the robot, the associated observations 
within the world $W$ are related via the
preimages of $X'$ under $h$, denoted by
$\inv{h}(X')\defeq\left\{\ell \in
Y(W)\mid h(\ell)\cap X'\neq \emptyset \right\}$.  
Below, the notation for a sensor map $h$ and its preimage $\inv{h}$ is extended
in the usual manner to p-graphs by applying the function
to labels on each observation edge, i.e., in the obvious way.

\subsection{Planning problems and plans}

\begin{definition}[planning problem and plan]
A \defemp{planning problem} is a p-graph $W$ along with a goal region $V_{\goal}\subseteq
V(W)$; a \defemp{plan} is a p-graph $P$ equipped with a termination region
$V_{\term}\subseteq V(P)$. 
\end{definition}

%Full definition
\begin{definition}[solves under sensor map]
\label{def:solves}
A plan $(P,V_{\term})$ \defemp{solves a planning problem $(W,V_{\goal})$ under
sensor map $h$} if there is some integer which bounds the length of all
joint-executions of $W$ and $\inv{h}\langle P\rangle$, and for each
joint-execution and any pair of nodes $(v \in V(\inv{h}\langle P\rangle),w \in
V(W))$ reached by that execution simultaneously, the following holds:
\begin{tightenumerate}
\item if $v$ and $w$ are both action nodes and, for every label borne by each
edge originating at $v$, there exist edges originating at $w$ bearing the same
action label; \item if $v$ and $w$ are both observation nodes and, for every
label borne by each edge originating at $w$, there exist edges originating at
$v$ bearing the same observation~label; 
% had to be sneaky to stop it wrapping a line
\item if $v \in V_{\term}$ and then $w \in V_{\goal}$;
\item if $v \notin V_{\term}$ then some extended joint-execution exists,
continuing from $v$ and $w$, reaching the termination region.
\end{tightenumerate}
\end{definition}
Property 1) and 2) define a notion of safety; 3) of correctness; 4) of liveness.
Note that the sensor map, modeling robot's sensor in this paper, may affect the
solvability of the planning problem. In other words, we have to examine the
safety when searching for sensor maps.

\subsection{Sensor design in a planning problem}
Now we can define the central problem of the paper:
\ourproblem{\textbf{Joint-Plan-Sensor-Design} (\dps)}
{A planning problem $(W, V_{\goal})$}
{All the sensor maps $\mathcal{H}$, such that there
exists a plan $(P, V_{\term})$ to solve the planning problem $(W,
V_{\goal})$ under each sensor map $h\in \mathcal{H}$.}
\vspace*{-8pt}

\section{Computational abstractions for sensor maps}
% We need to make some transition.
Sensor maps map observations to their images, while the planning problem is
defined in the preimage space. To solve this problem, we will begin by
considering an alternate form in the preimage space for the sensor maps.
%covers, and then defines operations on such covers.

\subsection{Equivalent representation for sensor maps}
%The sensor map plays an important role in checking solutions for planning
%problems and constructing the compatible world states for the observer. It
%encodes the ``structure" for the observations in the preimage space $Y$.
Any sensor map has an equivalent cover representation.
%In our previous paper~\cite{zhang18discreteplan}, the information disclosure
%policy partitions the preimages into equivalence classes, so as to make events
%always indistinguishable for the observer. In this paper, the label map divides
%the preimages into a set of overlapped sets. Each set is the set of preimages of
%some reading (realization), i.e., the set of possible observations or actions
%under some reading or realization. The events in the same set are
%indistinguishable under the corresponding reading (realization).  One
%observation (action) can be inferred (generated) from multiple readings
%(realizations). Here, we will interpret this conditional conflation
%as a collection of subsets that span over all events --- a cover, instead of a
%partition.

\begin{theorem}
For planning problem $(W, V_{\goal})$, 
any sensor map~$h$ is equivalent to a cover up to plan solvability.
\end{theorem}
\begin{proof}
$\Rightarrow:$ Given any sensor map $h$, 
to see whether a plan is a solution (cf. Def.~\ref{def:solves}), 
we must determine the preimage
$\inv{h}(x)=\{\ell\in Y(W)\mid h(\ell)=x\}$ for single readings~$x$. 
Collect all the data associated with $h$, on the $X$, via

\vspace*{-10pt}
\[M=\{\inv{h}(x_1), \inv{h}(x_2), \dots, \inv{h}(x_n)\},\]
\vspace*{-18pt}

\noindent where $X=\{x_1, x_2, \dots, x_n\}$.  
This is a multiset. But now observe that
where for any $x_i$ and $x_j$ we have 
$\inv{h}(x_i)=\inv{h}(x_j)$, we can
construct a new sensor map by replacing $x_i$ and $x_j$ with a new symbol $x'$.
This new sensor map is also a solution if and only if $h$ is a solution for 
\dps. Under this new sensor map, no two readings in the sensor map
share the same preimage, and $\inv{h}$ can be thus represented as 
set

\vspace*{-10pt}
\[C=\{\inv{h}(x_1), \inv{h}(x_2), \dots, \inv{h}(x_n)\},\]
\vspace*{-18pt}

\noindent where $\cup_{x_i\in X}\inv{h}(x_i)=Y(W)$. The set above is called a \defemp{cover} for set $Y(W)$. 
Henceforth, we call the cover for sensor map $h$ an \defemp{observation cover}, 
denote it $\coverintr{h}$. (It is a subset of the powerset of $Y(W)$,
i.e., $\coverintr{h}\subseteq \powerset{Y(W)}\setminus
\{\emptyset\}$.)
% You have the word powerset above, so the definition is wrapped in already.
% and $\powerset{Y}$ denotes the powerset of $Y$. 

$\Leftarrow:$ Having just showed that there exists a cover interpretation for
any sensor map $h$, we now construct a sensor
map for any observation cover. 
Suppose cover $\{S_1, S_2, \dots, S_k\} \subseteq \powerset{Y(W)}$ for set $Y(W)$
is given. Taking the first $k$ natural numbers for $X$,
consider a label map $h$ defined so that $y\overset{h}{\mapsto}\big\{i\in
\{1,2,\dots, k\} \mid y\in S_i\big\}$.

\medskip

Together, the cover $\coverintr{h}$ is an equivalent representation for any
sensor map $h$, up to plan solvability.
\end{proof} 

\subsection{Operations on observation covers}
Next, we give two operations on covers (projection and intersection) that are
useful for sensor maps.

The sensor map is a cover for all observations in the planning problem. 
Only some small number of observations may be applicable while
at particular world states.
We are interested in how the observations in such a reduced set conflate with
each other. This is realized via an operation that reduces the domain:
\begin{definition}[cover projection]
For cover $C=\{G_1, G_2, \dots, G_n\}$, denote its domain by $\domain{C}=\cup_{1\leq i\leq n}
G_i$. Then the \defemp{projection of $C$} on any domain $D$ is 
$\proj{C}{D}=\{G_i\cap D| G_i\in C\}$.  
\end{definition}  
% Yulin:I removed 'smaller' because D is not necessarily smaller. It can has
% some overlap with the \domain{C}. But it is true that after
% projection, the domain of the projected cover is smaller.
We call sensor map {$\proj{C}{D}$} with reduced domain \mbox{$\domain{C}\cap
D$} a \defemp{partial sensor map}. The word `partial' is apt as the sensor map 
need not cover every observation in the planning problem.

\smallskip

On the other hand, we are also interested in finding all sensor
maps with certain behavior on their restrictions. 
Specifically, we desire to find all label maps which, when given two partial label maps, agree with those label maps on their projections.
%\textcolor{red}{On the other hand, we are also interested in finding all sensor
%maps, each of which can generate two given partial label maps when projected to
%their domains.} 
%On the other hand, we are also interested in finding all sensor maps that can
%generate two partial ones when projected to their domains. 
This comes from an intersection between two partial sensor maps.
\begin{definition}[cover intersection]
For any two partial sensor maps, expressed as cover $C_1$ and $C_2$, with the
union of their domains $D=\domain{C_1}\cup\domain{C_2}$, then 
let $\mathbb{D}$ be all covers\footnote{Throughout, variables in blackboard
bold represent a list of covers.} whose domain is $D$.
Then the \defemp{intersection} of $C_1$ and $C_2$, denoted $\merge{C_1}{C_2}$, is
defined so that $\forall C'\in \mathbb{D}$, we have $C'\in \merge{C_1}{C_2}$, if and only if
\begin{itemize}
\item [{\footnotesize\emph{(a)}}]$\domain{C'}=\domain{C_1}\cup\domain{C_2}$, and
\item [{\footnotesize\emph{(b)}}]$\proj{C'}{\domain{C_1}}\subseteq C_1$ and $\proj{C'}{\domain{C_2}}\subseteq C_2$.
\end{itemize}
\end{definition}
Note that $\merge{}{}$ is associative and that
$\merge{C_1}{\emptyset}=\emptyset$ for any cover $C_1$. When no cover
that satisfies \emph{(a)} and \emph{(b)} above, then $\merge{C_1}{C_2}=\emptyset$. We
say that $C_1$ is \defemp{compatible} with $C_2$ if $\merge{C_1}{C_2}\neq
\emptyset$. We will also lift this notation to the intersection of lists
of covers. 
In writing $\merge{\mathbb{L}_1}{\mathbb{L}_2}$ 
for two lists of covers $\mathbb{L}_1$ and $\mathbb{L}_2$,
we mean $\merge{\mathbb{L}_1}{\mathbb{L}_2}=\cup_{C_1\in \mathbb{L}_1,
C_2\in\mathbb{L}_2} \merge{C_1}{C_2}$.

\section{Jointly searching for sensor designs \& plans}
First, we construct a robot's belief tree and then give
approaches to search for all sensor designs and plans in it.

\subsection{The belief tree under different sensor maps and actions}

The robot's plan must manage uncertainties owing to  initial ignorance, action
non-determinism, and sensor imperfection.  The robot's belief expresses this
uncertainty, which we represent as a set of states. Without this, the robot may
violate plan safety by trying to execute some action that is not possible in
its actual state.  The dynamics of the belief will be captured by a finite tree
structure, where each
vertex lists a set of world states, the robots' belief.
Plans need only visit each belief
vertex at most once.

\begin{theorem}
\label{thm:plan_kernel}
Let $\mathcal{W}$ be the set of estimated world states for the robot's belief.
For any sensor design $h\in \mathcal{H}$, where $\mathcal{H}$ is a set of
sensor maps for \dps, if there exists a plan that solves the planning problem
under $h$, then there exists another plan, also a solution, that visits $\mathcal{W}$ at most once
under $h$. 
\end{theorem}
\begin{sproof}
Construct a plan by shortcutting directly to the last visit 
to $\mathcal{W}$ under the same sensor map.
\end{sproof}
\begin{comment}
This theorem can be proved by constructing a new plan, which always takes the
action chosen at the last visit at $\mathcal{W}$ under the same sensor map.
Then the new plan is a shortcut of the original one. Inherited from the
original plan, the new one will always terminate at the goal region. 
\end{comment}
%\begin{proof}
%Let $h_y$ be one sensor design for \dps, the corresponding actuator design and
%plan be $h_u$ and $P$. Suppose $P$ visited $\mathcal{W}$ $n$ times. Let the
%action taken at $i$-th visit be $a_i$. Then we can construct a new plan $(P',
%V_{\term})$ which always take $a_i$ at $\mathcal{W}$. As a shortcut of plan $(P,
%V_{\term})$, plan $(P', V_{\term})$ always satisfies the stipulations, since it
%never creates additional I-states for the observer. In addition, $P'$ will also
%terminate at the goal region if $P$ does.
%\end{proof}

Let $A(w)$ be the set of outgoing events for vertex \mbox{$w\in V(W)$}. Then the belief
tree, a sketch of which appears in Fig.~\ref{fig:belieftree}, can be constructed as follows:
\begin{itemize}
\item [$\triangleright$]\emph{Initialization}: An initial vertex $\mathcal{W}^0$ of the same
vertex type is created for the set of initial world states $V_0(W)$. 
%(By minor abuse, we will refer to $\mathcal{W^0}$ as either a belief vertex or the set of world states $V_0(W)$.)

\item [$\triangleright$]\emph{Expanding action vertex $\mathcal{W}$}: 
Collect the common actions as $U({\mathcal{W}})=\cap_{w\in \mathcal{W}} A(w)$, i.e., the set of actions each of which is available at every state in $\mathcal{W}$.
Now, for any action $a\in U({\mathcal{W}})$,
consider the
transition $\mathcal{W}\xrightarrow{\{a\}}{} \mathcal{W'}$. 
If the set of world states $\mathcal{W}'$ has not appeared earlier in
the path from $\mathcal{W}^0$ to $\mathcal{W}$, add new belief vertex
$\mathcal{W'}$ connected via an edge bearing $\{a\}$. Otherwise, 
add transition from $\mathcal{W}$ to a vertex $\mathcal{W}_{\dummy}$
to avoid expanding the same belief vertex multiple times.

\item [$\triangleright$]\emph{Expanding observation vertex $\mathcal{W}$}: 
Let all possible observations at the states in $\mathcal{W}$ be
$Y({\mathcal{W}})$, i.e., $Y({\mathcal{W}})=\cup_{w\in \mathcal{W}} A(w)$.  As
before,  construct a transition from $\mathcal{W}$ to $\mathcal{W'}$  if
$\mathcal{W'}$ is new, or to $\mathcal{W}_{\dummy}$ otherwise.
But now do this, not just the singletons, but 
for every $G\subseteq Y({\mathcal{W}})$.

\item [$\triangleright$]\emph{Goals in the tree}: Mark $\mathcal{W}$ a goal state, if
$\mathcal{W}\subseteq V_{\goal}$.
\end{itemize}

%Starting with initial vertex denoting belief $\mathcal{W}$, we can construct the
%robot's belief tree. If $\mathcal{W}$ is a set of action states, we will obtain
%the set of actions $U_{\mathcal{W}}$ that are safe for every state in
%$\mathcal{W}$, i.e., $U_{\mathcal{W}}=\cap_{w\in \mathcal{W}} A(w)$, where
%$A(w)$ is the set of actions available at state $w$. Now we will consider the
%transition $\mathcal{W}\xrightarrow{\{a\}}{} \mathcal{W'}$ for any action $a\in
%U_{\mathcal{W}}$. If the set of world states $\mathcal{W}'$ is not visited in
%the ancestors, then we will construct an outgoing edge bearing $\{a\}$ to a new
%belief vertex denoting $\mathcal{W'}$. Otherwise, we will discard this
%transition. If $\mathcal{W}$ is a set of observation states, we will obtain the
%set of all observations $Y_{\mathcal{W}}$ available at each state in
%$\mathcal{W}$, i.e., $Y_{\mathcal{W}}=\cup_{w\in \mathcal{W}} O(w)$, where
%$O(w)$ is set of observations available at state $w$.  Similarly, we will
%construct a new vertex $W'$ and an outgoing edge
%$\mathcal{W}\xrightarrow{G}{}\mathcal{W'}$ for every $G\subset Y_{\mathcal{W}}$,
%if $\mathcal{W'}$ is not visited in its ancestors. We will repeat this process
%to expand each belief vertex $\mathcal{W}$ until $\mathcal{W}\subseteq
%V_{\goal}$, where we will mark the belief vertex $\mathcal{W}$ as a goal state.
%An example of the belief tree is shown in Fig.~\ref{fig:belieftree}. There are
%only a finite number of belief vertices in the tree, since the maximum depth of
%the belief tree is $2^{|V(W)|}$. 

\oldtext{The belief tree is of finite size, since the maximum depth of the belief tree is $2^{|V(W)|}$. }
The belief tree is finite.
Any sensor map and goal-achieving plan are a subtree that satisfies
the following:
\begin{itemize}
\item[($i$)] \emph{Goals are achieved:} the leaf vertices in the subtree are all in the goal region;
\item[($ii$)] \emph{Readiness to receive all observations:} the outgoing labels at a particular observation vertex in the subtree cover all
outgoing events in the original belief tree; 
\item[($iii$)] \emph{Discernment is consistent:}
the subset of observations in the tree is universal, i.e., if $\{o_1,
o_2\}$ appears on any edge of the subtree, then it will appear at every belief
vertex whose outgoing events contains both $o_1$ and $o_2$. 
\end{itemize}

\begin{figure}[b]
\vspace*{-8pt}
\centering
	\includegraphics[scale=0.34]{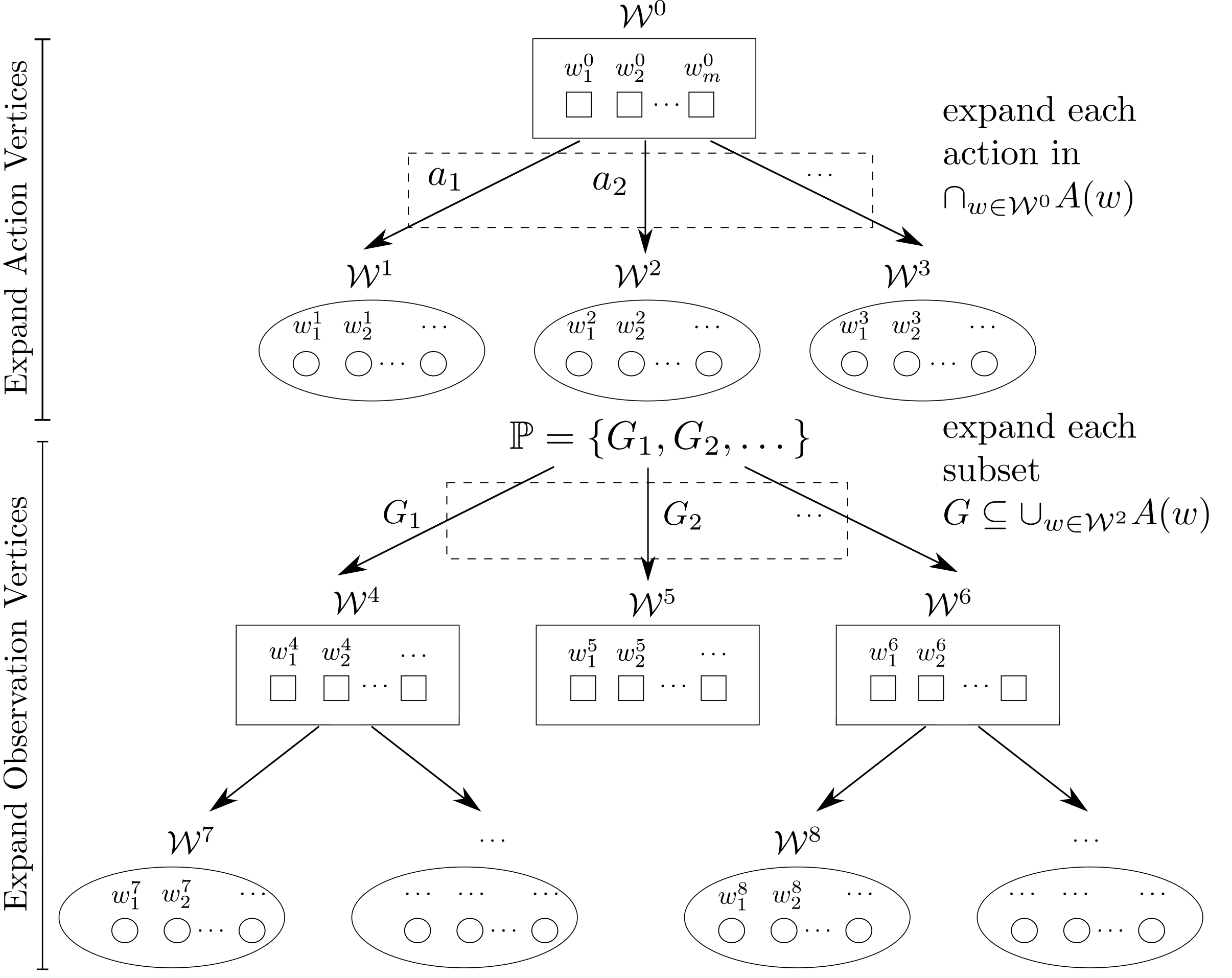}
        \caption{The robot's belief tree. Action and observation
vertices, visualized as boxes and circles respectively, have different expansions.}
        \label{fig:belieftree}
\end{figure}

\subsection{Searching for sensor designs and plans jointly}
Next, we search this structure for sensor designs and
plans jointly, returning all appropriate sensor maps. 
While the tree is constructed from the root down, this search bubbles from
the leaves back upwards.

For each belief vertex $\mathcal{W}$, we will maintain a list of covers,
denoted by $\lcw{\mathcal{W}}$, to record all the appropriate observation
covers in the subtree. When $\mathcal{W}$ is in the goal region, there are no
constraints on sensor maps from its subtree. 
Hence, we create a new symbol $\epsilon\not\in Y$, and initialize its cover
list to $\lcw{\mathcal{W}}=\llbracket\{\epsilon\}\!\rrbracket$. This will make it compatible with
any cover when integrating with the goal-achieving sensor covers in a bottom-up
manner.
%We initialize $\lcw{\mathcal{W}}$
%as the powerset of $G$, where $G$ is the set of observations on the
%\textcolor{orange}{latest} edge leading to $\mathcal{W}$ from $\mathcal{W}^0$.
%(If no $G$ can be found to
%initialize $\lcw{\mathcal{W}}$, then $\mathcal{W}^0$ transitions to goal state
%$\mathcal{W}$ with some action, and the planning problem can be solved without
%sensors.) 
For any non-goal belief vertex $\mathcal{W}^p$ (`p' stands for
parent), we will construct its cover list from its
children. Let the outgoing events be $\{G_1, G_2, \dots, G_m\}$ and the
corresponding child vertices be $\{\mathcal{W}^c_1,
\mathcal{W}^c_2, \dots, \mathcal{W}^c_m\}$ (`c' for child). Then we have:
\begin{tightitemize3}
\item If $\mathcal{W}^p$ is an action vertex, then each cover in any of its children's lists
$\lcw{\mathcal{W}^c_i}$ is a valid one for $\mathcal{W}^p$ (under a
particular action choice), i.e., $\lcw{\mathcal{W}^p}=\cup_{1\leq i\leq m}
\lcw{\mathcal{W}^c_i}$.

\item If $\mathcal{W}^p$ is an observation vertex, 
we must consider the combinations from $\{G_1, G_2, \dots, G_m\}$ that
nevertheless cover $Y(\mathcal{W}^p)$.
Let $K$ denote one such combination, 
then $K=\{G_{k_1}, G_{k_2}, \dots, G_{k_\ell}\}$, where
$k_j\in \{1,2,\dots, m\}$ and $\cup_{1\leq j\leq \ell} G_{k_j}=Y(\mathcal{W}^p)$.
Each edge labeled with $G_{k_j}$
gives a child vertex $\mathcal{W}^c_{k_j}$, where that child has a cover list
$\lcw{\mathcal{W}^c_{k_j}}$ modeling the sensors that can reach the goal from
$\mathcal{W}^c_{k_j}$. For a given combination $K$, representing a set of sensor
readings, we want to find all sensor maps, denoted as $\lc{K}$, that can
generate $K$ when projected to $\domain{K}$, and is goal-achieving for
the subtrees starting from each child vertex $\mathcal{W}^c_{k_j}$. This is
realized via intersection operations: 
%We want to find all covers, comprising cover list $\lc{K}$ under
%combination $K$, which produces not only $K$ when projected to $\mathcal{W}^p$,
%but also at least a cover from each $\lcw{\mathcal{W}^c_{k_j}}$ when projected to
%their corresponding domain.  
%$\lc{K}$ is computed as follows:

\vspace*{-17pt}
$$\lc{K}=\cup_{C_1\in\lcw{\mathcal{W}^c_{k_1}}, \dots, C_n\in
\lcw{\mathcal{W}^c_{k_n}}}\merge{K}{\merge{C_1}{\dots\merge{C_{m-1}}{C_m}}}.$$
\vspace*{-17pt}

The $\merge{}{}$ operation guarantees the universality of the subsets in the resulting cover.
Let $\mathbb{C}$ be the set of all such combinations, such that their labels cover
$Y(\mathcal{W}^p)$. Each combination $K\in\mathbb{C}$ gives a list of covers for
the parent vertex. So we update $\lcw{\mathcal{W}^p}$ to value
$\cup_{K\in \mathbb{C}} \lc{K}$.
\end{tightitemize3}

By propagating the list of covers from the goal vertices back upward until the
initial belief vertex, we are able to obtain all the covers from
$\lcw{\mathcal{W}^0}$ where there exists some plan for each cover in
$\lcw{\mathcal{W}^0}$ toward the goal.

\subsection{Compact representation with upper covers}

In the data structure above, we need to maintain a list of covers
$\lcw{\mathcal{W}}$ for each belief vertex $\mathcal{W}$. The list
can grow very large. 
Luckily, we only need to maintain the largest covers among the ones with the
same domain. 
Every subset of such covers is also a valid solution, so long as it is a proper cover.

\begin{theorem}
If $C$ is an observation cover in the solution of \dps, then for any
$C'\subseteq C$, such that $\domain{C'}=\domain{C}$, there exists a plan
achieving the goal. 
\label{thm:subset_close}
\end{theorem}
This theorem can be proved by showing that the subtree without edges bearing
subsets of events in $C'\setminus C$, still has all leaf vertices as goals and
sensor map as a valid cover.

\begin{definition}[upper cover]
\label{def:upper}
Let $\mathbb{C}$ be a list of observation covers, $C$ is an \defemp{upper cover} in
$\mathbb{C}$ if there does not exist any cover $C'\in \mathbb{C}$, such that
$C'\supsetneq C$ and $\domain{C'}=\domain{C}$.
\end{definition}

According to Theorem~\ref{thm:subset_close}, we only need to maintain a set of
upper covers in each $\lcw{\mathcal{W}}$.

\subsection{Empirical explorations of sensor maps} 
%\textcolor{orange}{In this problem, I did not model the heading information intot the planning problem for simplicity. The robot is not allowed to turn and go back.}
%\textcolor{blue}{Then what does `constructing a sensor map' mean? How does it magically give things that make sense in English? Why does the finest sensor not distinguish heading, to give GPS? This doesn't make sense to me. Write what you did, using as many words as needed to make the statement crystal clear, I can compress it down.}
We implemented the algorithms in Python to search for all sensor map solutions
for the problem displayed in Fig.~\ref{fig:scenario_1},
a modified version of the 
(Fig.~\ref{fig:motivating_example}) motivating example: a robot,
initially located at $1$ or $2$, moves to the charging station.  The
robot can only move forward one or two steps, turn left or right at the
location $5$ or the corner $6$. The robot must avoid bumping into the walls of
the four offices $A$--$D$ and also the stairs. It must, thus, obtain
information from its sensors to reduce its uncertainty.  To realize this
scenario, we construct a world p-graph with \num{22} states and
\num{11} observations. 

The algorithm outputs an upper cover with \num{767} entries. By enumerating all
subsets of the upper cover that covers all the observations
(Theorem~\ref{thm:subset_close}), an enormous number of sensor maps are
produced.  Among these, several are directly recognizable sensors. For example,
they include a sensor map distinguishing every pair of positions, describing a
{\sc gps} device. The sensor partitioning the situations into those before and after
bumping into walls, could be realized as a contact sensor. 

Naturally, some of the sensor maps are inscrutable and there are others for
which not known hardware implementation could be discerned.  For instance, the
sensor isolating cell $5$ when facing west from cell $7$ when facing north (e.g., a
distance sensor won't work).  This motivates the next section.

%Two additional sensor maps are
%shown as follows: 
%\begin{itemize}
%\item A sensor map which distinguishes all positions, is able to be realized
%as a GPS sensor. 
%%The GPS sensor only gives information about robot's coordinate but not its
%%heading. 
%
%%\item A distance sensor helps to resolve the initial uncertainty since they have
%%different distances toward the wall. The angular information can be inferred
%%from the distance sensor, since it gives different distance readings when the
%%robot is heading north and west.
%\item A bumper sensor can also resolve the uncertainty by driving the robot into
%a situation where the robot will bump into a wall at one position while it does
%not at another position. 
%\end{itemize}
%However, in this large set, there exist sensor maps that cannot be realized by
%some known hardware. For example, the sensor distinguishing cell $5$
%when facing west from cell $7$ when facing north cannot be realized as a
%distance sensor, since they are the same distance to the wall. This
%motivates the next section.
\begin{figure}[t]
\centering
	\includegraphics[scale=0.5]{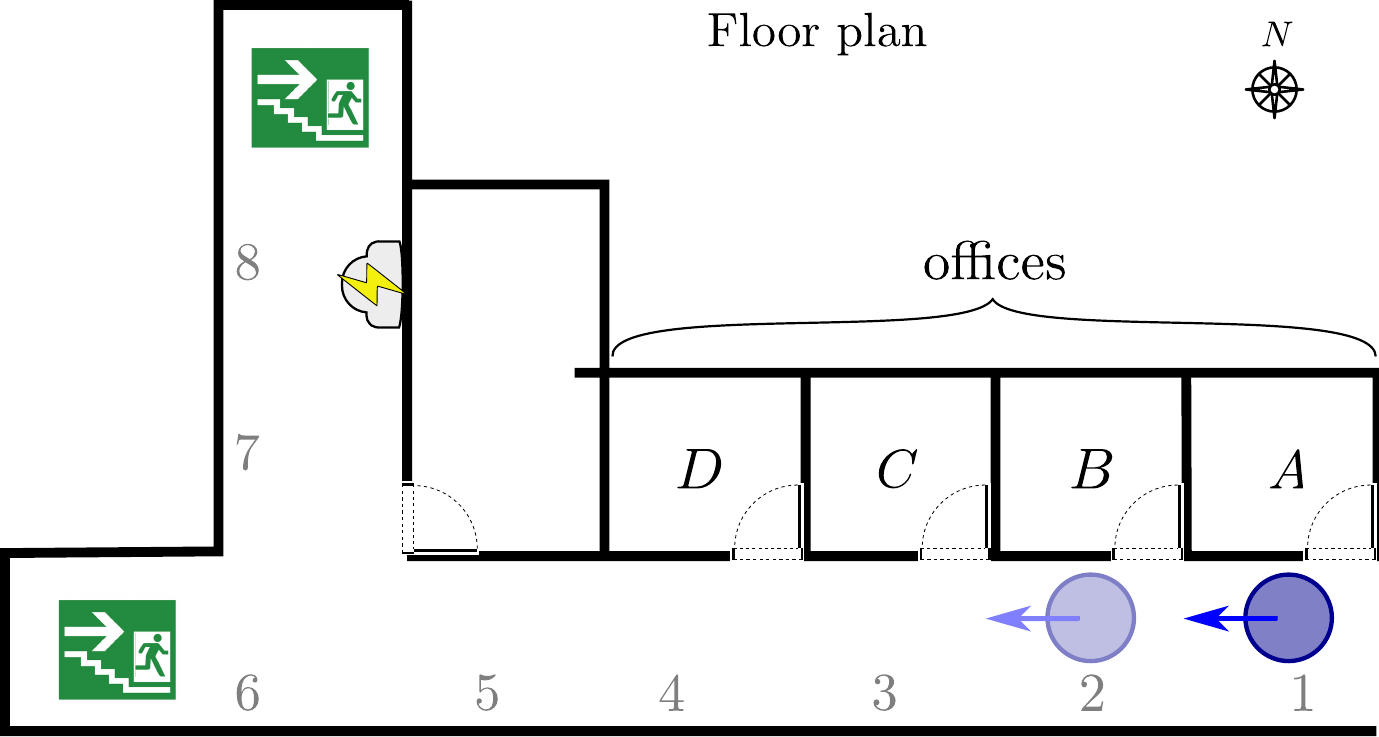}
	\includegraphics[scale=0.4]{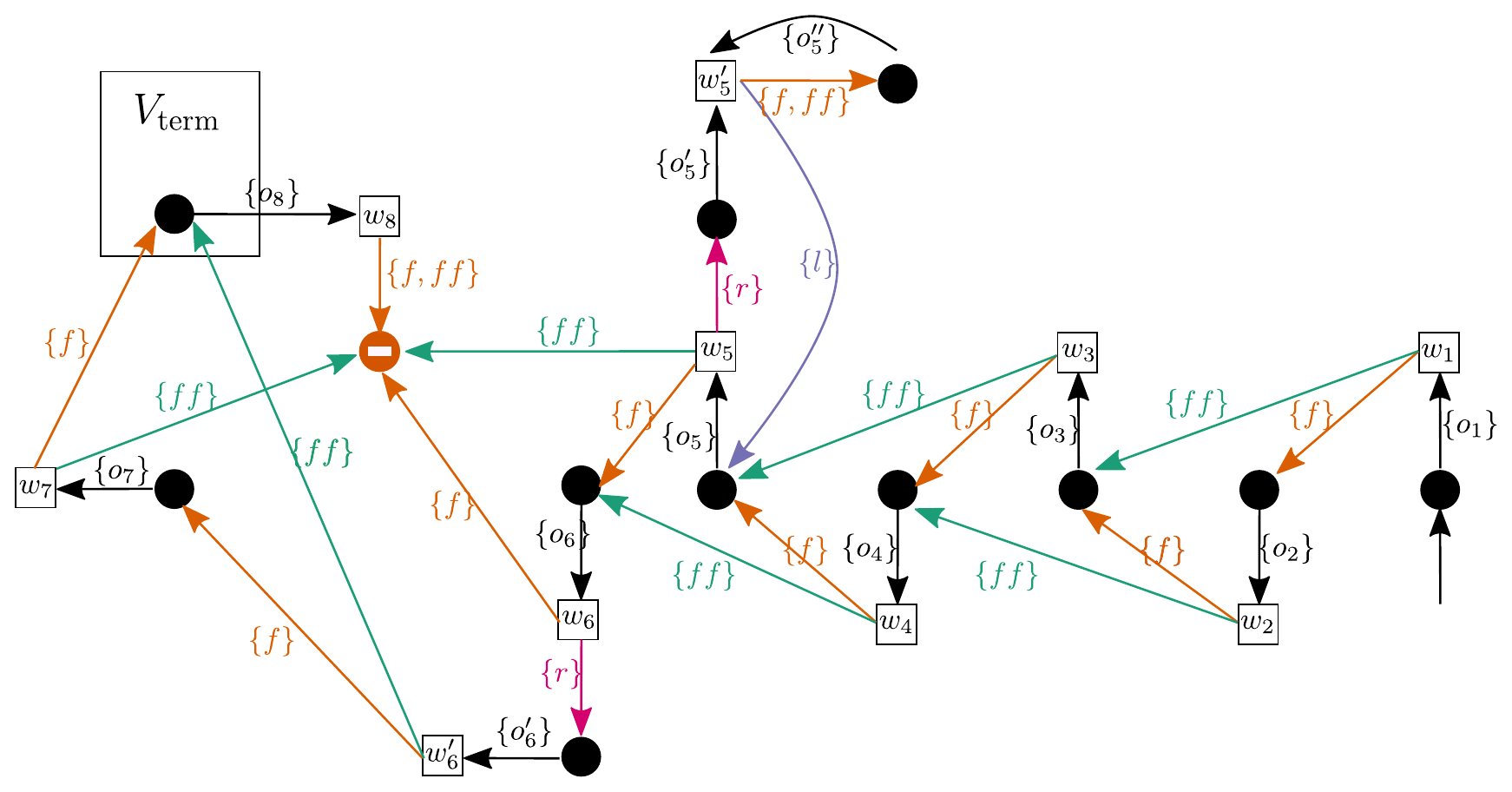}
	\caption{A robot moves toward the charging station, while avoiding
stairs. The figure below shows the p-graph of this planning problem.}
        \label{fig:scenario_1}
\end{figure}

\section{Structure and fabrication constraints for realizable sensors}

The covers found via the preceding approach might be thought of as a sort of
`free object', on which we may now impose additional constraints. Specifically
we're interested in including constraints that help model aspects pertinent to
realizable sensors.

\subsection{Sensor map properties}
We will start with the following properties:
\begin{property}[Partition]
Cover $C=\{G_1, G_2, \dots, G_n\}$ is a \defemp{partition}, if $G_i\cap G_j = \emptyset$ for any $i,j\in\{1,\dots, n\}$, $i \neq j$.
\end{property}
The label map for the camera in Fig.~\ref{fig:motivating_example} is a
partition, as it divides the space into red, gray and white locations.

\begin{comment}
% This just says the same thing twice.
\begin{definition}[Neighbor]
Given a reflexive and commutative relation $N\subseteq Y\times Y$, we say $y_1$
is a neighbor of $y_2$ if $(y_1, y_2)\in N$, and write it as $y_1N y_2$.
\end{definition}
The neighboring relation is also reflexive and commutative.
\end{comment}

The next concept of interest is a notion of contiguousness, but we need
a more basic structure first.

\begin{definition}[Neighbor]
Relation $N\subseteq Y\times Y$,
written $y_1 N y_2$, is a 
\defemp{neighbor relation} if it is 
reflexive and commutative. %We will write it as $y_1 N y_2$.
\end{definition}

\begin{property}[Contiguous]
With neighbor relation $N$, then
$\mathcal{C}$ is the largest contiguous cover if 
(1)~$\forall y\in
Y$, $\{y\}\in \mathcal{C}$; 
(2)~$\forall G_1, G_2\in \mathcal{C}$, $G_1\cup G_2\in
\mathcal{C}$ \gobble{if and only if} $\iff$ $\exists y_1\in G_1, \exists y_2\in G_2$, such that
$y_1 N y_2$. A given cover $C$ is \defemp{contiguous}, if $C\subseteq \mathcal{C}$.
\end{property}
The distance sensor in Fig.~\ref{fig:motivating_example} has a contiguous
sensor map for the obvious neighbor notion, since its noise distribution is contiguous.

\begin{property}[output]
Cover $C$ is \defemp{$k$-outputting}, if \mbox{$|C|=k$}.
\end{property}
The cardinality of the sensor keeps track of total number of output readings,
a sort of notion of dynamic range.
\begin{property}[overlap]
A cover $C=\{G_1, G_2, \dots, G_n\}$ is \defemp{$k$-overlapping}, if $\forall i,
j\in\{1,\dots, n\}$ and $i\neq j$, $|G_i\cap G_j|\leq k$.
\end{property}
%The overlapped events for different readings are indistinguishable from these two readings. 
This is a generalization of the partition property, quantifying how 
much readings bleed into one another.

\begin{property}[width]
Cover $C=\{G_1, G_2, \dots, G_n\}$ is \defemp{$k$-wide} (or, has width $k$), if $\forall 1\leq i\leq n$,
$|G_i|=k$.
\end{property}
%The width of a cover models the volume of the noise, which is the total number of events could potentially happen for a single sensor reading.

The width of a cover gives a notion of precision, a sense of the
volume of noise, describing the number of events that could 
account for a single sensor reading.

The properties above may also be combined in specifying constraints on 
sensor maps.
%Cover -> partition -> projection 
%convex covers, convex partition 
%convex partitions
%max 1 intersection convex cover (overlap)
%cardinality
%matrix on subset size

All of the properties can be used either in (1)~reducing the sets generated, or
in (2)~filtering to discard those which violate the constraints, as operators
are applied.  For instance, in the first case, if searching for partitions
only, then partitions exclusively need be computed---a process easier to write
and faster to execute than the full cover case. 

\subsection{Empirical search for sensors under fabrication constraints}

We included the properties described above in our implementation and examined
in the following scenario. A robot moves along a cyclic track toward some goal,
marked by a star. The robot can move forward or backward at different speeds
at different parts of the track, which discretizes the track into $6$ segments
$\{s_1, s_2, \dots, s_6\}$ as shown in Fig.~\ref{fig:scenario_2}.  The angular
range of segment $s_i$ is denoted as $\{o_i\}$ for $i\in\{1,\dots, 4\}$, and
$\{o_i, o\}$ for $i\in\{5,6\}$. The overlap $o$ is the common angular range for
both $s_5$ and $s_6$, arising from the kink. Now, the set of all observations
is $Y=\{o, o_1, o_2, \dots, o_6\}$, where each observation represents a range
of angles.\footnote{Previously we pointed out that $Y$ was finite; this is
still true, though the elements it contains are themselves infinite sets.} The
neighbor relationship of these observations inherits from the circular
neighbor relationship of their angles as shown in the figure.  Considering only
forward or backward actions,  the robot, initially located at $s_1$ or $s_3$,
must move to reach the goal $s_5$. To achieve this, the robot has to reduce its
uncertainty, and it does this via a VHF omnidirectional range (VOR) sensor. As
shown on the right-hand side of Fig.~\ref{fig:scenario_2}, the sensor measures
the angular information via a timer. The specification of the timer determines
the properties of the sensor map. Suppose we have a timer with no noise, then it
gives $1$-overlapping contiguous sensor maps, such as $[\{o_2, o_3, o_4, o_5\},
\{o_5,o\}, \{o, o_6, o_1\}]$.  There are only \num{2183} such sensor maps.  But with
a noisy timer, it generates contiguous sensor maps, which leads to
\num{235807} observation covers.  

\begin{figure}[ht!] 
\centering
\includegraphics[scale=2]{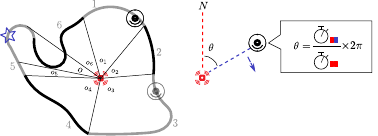} \caption{A robot with 
a sensor equipped to determine angles moves from
its initial position toward the goal along a cyclic track. The sensor is
realized by a VOR-like beacon at the center, a photo-electric sensor and a timer on
the robot. The beacon has a
unidirectional blue light rotating at a fast constant angular velocity, which
is so fast that can be neglected with respect to robot's movement. It also
emits an omnidirectional red light when the blue light points North.  The robot
can determine angular information by timing the difference between seeing
red-red and red-blue flashes. \label{fig:scenario_2}} 
\end{figure}

\label{sec:noiseless}

\noindent Consideration of the scenario above leads to the following:
\begin{proposition}
A noiseless sensor taking measurements on a continuous or non-continuous space
always gives a $1$-overlapping sensor map under discretization.
\end{proposition}
\begin{proof}
When there is no noise for the sensor, the sensor map partitions the original
continuous or non-continuous measurement space. The task may only need a coarser
discretization of the measurement space. If every boundary of the sensor map
is a discretization boundary, then the sensor map is still a partition
on the discretized space. If it is not, then there exists a sensor map boundary
that falls into one of the discretized observations. That observation is 
shared by the preimage of only the readings separated by the corresponding
sensor map boundary. Hence, the maximum overlap between subsets in the
observation cover is $1$.
\end{proof}

\section{Generalization to belief stipulations}

Some prior work has examined instances wherein a robot should be stopped from
knowing too much due to privacy considerations~\cite{OKa08,zhang18complete}.  In
these cases, one may pose constraints on robot's belief; in our prior work this
was achieved via logical expressions \cite{zhang18discreteplan}. To search for
sensor maps and plans that also satisfy these richer stipulations, the algorithm
above needs the following modifications:
\begin{tightitemize2}
\item The belief tree should only contain belief vertices satisfying the
stipulations, and the dummy vertex. We transition to the
dummy if the target belief violates the stipulations. 

\item We must expand the action vertex in the belief tree over all subsets of
actions in the plan instead of just the singleton ones, since none of the
singleton actions may transition to the subtree that satisfies the stipulations
in each belief vertex.
\end{tightitemize2}

\section{Related work}
% Another thread is about planning with goal obfuscation. They map the robot's
% state action pair (s,a) to a single observation received by the observer.
% There is a similar notion of label map, which is a partition. 

%kulkarni2018resource
%Anagha Kulkarni, Matthew Klenk, Shantanu Rane, Hamed Soroush, Resource
%Bounded Secure Goal Obfuscation, AAAI Fall Symposium on Integrating Planning,
%Diagnosis, and Causal Reasoning, 2018. 

%kulkarni2019unified
%Anagha Kulkarni, Siddharth Srivastava and Subbarao Kambhampati, A Unified
%Framework for Planning in Adversarial and Cooperative Environments, ICAPS
%workshop PlanRob, 2018.

%This paper describes ideas at the interplay of two 
%bodies research: techniques to automate design of robots, and methods
%for addressing privacy and associated security considerations.
Approaches for automated design of robots have been the subject of three recent
workshops at RSS and ICRA over the last \num{3} years~\cite{nilles18robot}. Current research examines
aspects of hardware fabrication (e.g., 3D-printing~\cite{Fuller02} and
prototyping~\cite{HoovFearing08,Fitzner17}), interconnection
optimization~\cite{ZiglarWW17}, rapid end-to-end development and
deployment~\cite{Luck-RSS-17,Schulz17}, automated synthesis (jointly for
mechanisms and controllers) from specifications of desired
capabilities~\cite{Mehta18}, and optimization subject to functionality--resource
interdependencies~\cite{censi17co,pervan2018low}.

A rich history of robotics research has examined the information required to
accomplish a particular task, including specifically what sensors ought
to provide \cite{erdmann95understanding}. Since sensors can be costly and
unreliable, important early papers explored how one might forgo them
entirely~\cite{erdmann1988exploration,Mason93kicking}; other work examined how
one might reason about sensors to establish that they do provide enough
information~\cite{donald95information,OKaLav08}. Our prior
work~\cite{zhang18complete} exploits properties of carefully
conceived sensors.  
We are interested in all possible sensors, including hypothetical ones, that
provide adequate information to solve the planning problem. 
%To reason about sensors, abstractions must be made to model the sensors or the information that sensors should provide. 
Imperfection in sensors is modeled as conflation in the perceived events. This
conflation is usually considered to be transitive in existing
work~\cite{OKaLav08,zhang18discreteplan,zhang2018does,ghasemlou2018delineating,kulkarni2019unified},
when reasoning about the information through the sensors. This
paper describes sensors via a sensor map representation which can model
non-transitive conflation, in the spirit
of~\cite{setlabelrss,saberifar18pgraph}. It contributes methods to search for
all sensors such that there exists a plan for each one to accomplish a given
task.

\section{Conclusion}
This paper explores the space of all sensors that provide enough information
to solve the planning problem for the robots. The abstraction used for sensors
is a generalization of prior models. A notion of upper cover is proposed to
compress the representation and speed the search process. Properties are
introduced to express domain-knowledge regarding
fabrication constraints for sensors.

%Abstractions and algorithms are
%proposed to model and search for all sensors. 
%The algorithms in this paper can help shed light on a task's particular
%informational demands, even if the requirements concern obscuring rather than
%exposing state and structure.  Prior work with robots that actively manipulate
%their ignorance used specific and somewhat unusual sensors.  One easily sees
%that those sensors have certain characteristics which are useful in controlling
%belief uncertainty (unbounded and contiguous preimage sets, for instance), but
%a detailed understanding of whether those attributes are necessary has been
%unclear. The approach we have developed enables the roboticist to explore such
%questions for specific problems, allowing the space of the designs and plans to
%be mapped, constraints to be visualized, and trade-offs evaluated.
%
%More work remains to be done to scale up the formulation we have
%presented to address problems of moderate and large size. We believe that
%both informed search and randomized algorithms may help in this regard.
%
%The former requires the definition of admissible heuristics over design
%costs, an intriguing question for which the monotonicity under extension
%property we have defined may play an important role.

\IEEEpeerreviewmaketitle

\newpage

\bibliographystyle{IEEEtran}
\bibliography{mybib}
\end{document}

%% file: defs.tex
\providecommand{\gobble}[1]{}

\providecommand{\term}{\ensuremath{\mathrm{term}}}
\providecommand{\goal}{\ensuremath{\mathrm{goal}}}
\newrobustcmd*{\mysquare}[1]{\tikz{\filldraw[draw=black,fill=#1] (0,0)
rectangle (0.2cm,0.2cm);}}

\newrobustcmd*{\mycircle}[1]{\tikz{\filldraw[draw=black,fill=#1] (0,0) circle [radius=0.1cm];}}

\newrobustcmd*{\mytriangle}[1]{\tikz{\filldraw[draw=black,fill=#1] (0,0) --
(0.2cm,0) -- (0.1cm,0.2cm);}}

\definecolor{amethyst}{rgb}{0.6, 0.4, 0.8}
\definecolor{gre}{rgb}{0.4, 1.0, 0.5} % green
\definecolor{gra}{rgb}{0.8, 0.8, 0.85} \definecolor{re}{rgb}{1.0, 0.5, 0.5}

\definecolor{setformula}{rgb}{0.0, 0.0, 0.0}
\definecolor{equivrel}{rgb}{0.0, 0.0, 0.0}

\newcommand*{\probleminternal}[4]{
	\par
	\medskip
	\noindent\fbox{\parbox{0.98\columnwidth}{
		\textbf{#4:} {#1} \\[0.05in]
		\renewcommand{\tabcolsep}{2pt}
		\begin{tabularx}{\linewidth}{rX}
			\emph{Input:} & #2 \\
			\emph{Output:} & #3
		\end{tabularx}
	}}
	\par
	\medskip
	\par
}

\let\emptyset\varnothing

\newcommand*{\ourproblem}[3]{\probleminternal{#1}{#2}{#3}{Problem}}
\definecolor{darkblue}{rgb}{0.15,0.09,0.3}

\definecolor{formula}{rgb}{0.08,0.35,0.08}

\definecolor{varcolour}{rgb}{0.0,0.0,0.0}

\newrobustcmd*{\true}{{\sl True}\xspace}
\newrobustcmd*{\false}{{\sl False}\xspace}

\DeclareFontFamily{OT1}{pzc}{}
\DeclareFontShape{OT1}{pzc}{m}{it}{<-> s * [1.10] pzcmi7t}{}
\DeclareMathAlphabet{\mathpzc}{OT1}{pzc}{m}{it}

\newcounter{tecounter}
\newenvironment{tightenumerate}
{
    \begin{list}{\arabic{tecounter}\addtocounter{tecounter}{1})}{%
    \setcounter{tecounter}{1}
        \setlength{\leftmargin}{08pt}
        \setlength{\topsep}{1pt}
        \setlength{\partopsep}{0pt}
        \setlength{\itemsep}{2pt}
        \setlength\labelwidth{0pt}}
        \ignorespaces}
{\unskip\end{list}}

\newenvironment{tightitemize2}
{
    \begin{list}{$\bullet$}{%
        \setlength{\leftmargin}{9pt}
        \setlength{\topsep}{-4pt}
        \setlength{\partopsep}{0pt}
        \setlength{\itemsep}{2pt}}
        \ignorespaces}
{\unskip\end{list}}

\newenvironment{tightitemize3}
{
    \begin{list}{\hspace*{-3pt}$\bullet$\hspace*{-2pt}}{%
        \setlength{\leftmargin}{8pt}
        \setlength{\topsep}{2pt}
        \setlength{\partopsep}{0pt}
        \setlength{\itemsep}{2pt}}
        \ignorespaces}
{\unskip\end{list}}

\usepackage[noend]{algorithmic}
\algsetup{linenosize=\tiny}

\providecommand{\inv}[1]{\ensuremath{{#1}^{-1}}}

\newcommand{\defeq}{\ensuremath{\coloneqq}}

\providecommand{\reachings}[2]{\textcolor{setformula}{\ensuremath{\mathcal{S}^{#1}_{#2}}}}
\providecommand{\exactreachings}[2]{\textcolor{setformula}{\ensuremath{\mathbb{S}}^{#1}_{#2}}}

\newcommand\restr[2]{{% we make the whole thing an ordinary symbol
  \left.\kern-\nulldelimiterspace % automatically resize the bar with \right
  #1 % the function
  \vphantom{\big|} % pretend it's a little taller at normal size
  \right|_{#2} % this is the delimiter
  }}

\providecommand{\powerset}[1]{\mathcal{P}({#1})}
\providecommand{\coverintr}[1]{C_{#1}}
\providecommand{\dps}{{\sc JPSD}\xspace}
\providecommand{\lcw}[1]{\mathbb{L}(#1)}
\providecommand{\domain}[1]{d(#1)}
\providecommand{\proj}[2]{\pi_{#2}(#1)}
\providecommand{\lc}[1]{\mathbb{L}_{#1}}
\providecommand{\merge}[2]{#1\sqcap #2}

\providecommand{\defemp}[1]{\emph{#1}}
\providecommand{\dummy}{{\textrm{dummy}}}

\DeclareSymbolFont{stmry}{U}{stmry}{m}{n}
\DeclareMathSymbol\llbracket\mathop{stmry}{"4A}
\DeclareMathSymbol\rrbracket\mathop{stmry}{"4B}